\documentclass{article}

\usepackage{times}
\usepackage{graphicx}
\usepackage{subfigure}
\usepackage{natbib}
\usepackage{algorithm}
\usepackage[noend]{algorithmic}
\usepackage{hyperref}

\usepackage{booktabs}

\usepackage[accepted]{icml2018_arxiv}

\icmltitlerunning{Approximation Algorithms for Cascading Prediction Models}

\usepackage{amsmath,amssymb}
\usepackage{amsthm}

\newtheorem{definition}{Definition}
\newtheorem{lemma}{Lemma}

\newtheorem{theorem}{Theorem}

\newcommand{\overbar}[1]{\mkern 1.5mu\overline{\mkern-1.5mu#1\mkern-1.5mu}\mkern 1.5mu}

\newcommand{\card}[1]{\ensuremath{\left|#1\right|}}

\newcommand{\extreals}[0]{\ensuremath{\overbar{\mathbb{R}}}}
\newcommand{\indicator}[1]{\ensuremath{\mathbb{I}\left[#1\right]}}
\newcommand{\reals}[0]{\mathbb{R}}
\newcommand{\seq}[1]{\ensuremath{\left\{#1\right\}}}
\newcommand{\set}[1]{\ensuremath{\left\{#1\right\}}}
\newcommand{\surl}[1]{\begin{small}\url{#1}\end{small}}

\newcommand{\ac}[0]{a}
\newcommand{\accmodel}[0]{\hat q}
\newcommand{\ametric}[0]{q}
\newcommand{\amg}[0]{g}
\newcommand{\cost}[0]{c}

\newcommand{\examples}[0]{\mathcal{X}}

\newcommand{\labels}[0]{\mathcal{Y}}
\newcommand{\pairs}[0]{\ensuremath{\examples \times \labels}}
\newcommand{\predictionmodel}[0]{p}

\begin{document}

\twocolumn[
\icmltitle{Approximation Algorithms for Cascading Prediction Models}

\begin{icmlauthorlist}
	\icmlauthor{Matthew Streeter}{google}
\end{icmlauthorlist}

\icmlaffiliation{google}{Google, LLC}
\icmlcorrespondingauthor{Matthew Streeter}{mstreeter@google.com}
\icmlkeywords{cascades, machine learning}

\vskip 0.3in
]

\printAffiliationsAndNotice{}

\begin{abstract}
We present an approximation algorithm that takes a
pool of pre-trained models as input
and produces from it a cascaded model with similar accuracy
but lower average-case cost.  Applied
to state-of-the-art ImageNet classification models, this yields
up to a 2x reduction in floating point multiplications,
and up to a 6x reduction in average-case memory I/O.
The auto-generated
cascades exhibit intuitive properties, such as using lower-resolution
input for easier images and requiring higher prediction
confidence when using a computationally cheaper model.
\end{abstract}

\section{Introduction}

In any machine learning task, some examples are harder than
others, and intuitively we should be able to get away with less computation on
easier examples.  Doing so has the potential to reduce serving costs in the
cloud as well as energy usage on-device, which is
important in a wide variety of applications \cite{guan2017energy}.

Following the tremendous empirical success of deep learning,
much recent work has focused on making deep neural networks adaptive,
typically via an end-to-end training approach in which the network
learns to make example-dependent decisions as to which computations are
performed during inference.  At the same time, recent work on neural
architecture search has demonstrated that optimizing over thousands of
candidate model
architectures can yield results that improve upon state-of-the-art
architectures designed by humans \cite{zoph2017learning}.
It is natural to think that combining these ideas should lead to even
better results, but how best to do so remains an open problem.

One of the motivations for our work is that for many problems, there
are order-of-magnitude differences between the cost of a reasonably
accurate model and that of a model with state-of-the-art accuracy.
For example, the most accurate
NasNet model achieves 82.7\% accuracy on ImageNet using 23 billion multiplies
per example, while a MobileNet model achieves 70.6\% accuracy with only
569 million multiplies per example \cite{howard2017mobilenets, zoph2017learning}.  If we could identify the images on
which the smaller model's prediction is (with high probability) no less
accurate than the larger one's, we could use fewer multiplications on
those images without giving up much accuracy.

In this work, we present a family of algorithms that can be used to
create a cascaded model
with the same accuracy as a specified reference
model, but potentially lower average-case cost, where
cost is user-defined.
This family is defined by a meta-algorithm with various pluggable
components.
In its most basic instantiation, the algorithm takes a pool of
pre-trained models as input and produces a cascaded model in two steps:
\begin {enumerate}
	\item It equips each model with a set of possible rules for returning ``don't know"
		(denoted $\bot$) on examples where it is not confident.  Each (model, rule) combination is called an \emph{abstaining model}.
	\item It selects a sequence of abstaining models to try, in order,
		when making a prediction (stopping once we find a model
		that does not return $\bot$).
\end {enumerate}

We also present instantiations of the meta-algorithm that generate new
prediction models on-the-fly, either using lightweight training of ensemble
models or a full architecture search.  We also discuss a
variant that produces an adaptive policy tree rather than a fixed
sequence of models.

An important feature of our algorithms is that they scale efficiently
to a large number of models and (model, rule) combinations.
They also allow for computations performed by one stage of the cascade
to be re-used in later stages when possible (e.g., if two successive
stages of the cascade are neural networks that share the same first $k$
layers).

\section {Abstaining models} \label {sec:abstaining_models}

Our cascade-generation algorithm requires as input a set of \emph{abstaining models}, which are prediction models that return ``don't know" (denoted $\bot$)
on certain examples.
For classification problems, such models are known as \emph{classifiers with a reject option}, and methods for training them have been widely studied \cite{yuan2010classification}.
In this section we present a simple post-processing approach that can be used
to convert a pre-trained model from any domain into an abstaining model.

We assume our prediction model is a function $\predictionmodel: \examples \rightarrow \labels$, and
that its performance is judged by taking the expected value of an accuracy
metric $\ametric: \labels \times \labels \rightarrow \reals$, where
$\ametric(\hat y, y)$ is the accuracy of prediction $\hat y$ when the true
label is $y$.
Our goal is to create an abstaining model
$m: \examples \rightarrow \labels \cup \set{\bot}$
that returns $\bot$ on
examples where $\predictionmodel$ has low accuracy, and returns
$\predictionmodel(x)$ otherwise.

Toward this end, we create a model $\accmodel$ to predict $\ametric(\predictionmodel(x), y)$ given $x$.
Typically, this model is based on the values
of intermediate computations performed when
evaluating $\predictionmodel(x)$.  We train the model to estimate the value of the accuracy
metric, seeking to achieve
$\accmodel(x) \approx \ametric(p(x), y)$.
We then return $\bot$ if the predicted accuracy falls below some threshold.

As an example, for a multi-class classification problem, we might
use the entropy of the vector of predicted class probabilities as a feature,
and $\accmodel$ might be a one-dimensional isotonic regression that
predicts top-1 accuracy as a function of entropy.
The rule would then return $\bot$ on examples where entropy is too high.

Pseudo-code for the abstaining model is given in
Algorithm~\ref{alg:confident_abstaining_model}.  Here and elsewhere, we distinguish between an algorithm's
parameters and its input variables.  Specifying values for the parameters
defines a function of the input variables, for example ConfidentModel($\cdot; \predictionmodel, \accmodel, t)$ denotes the abstaining model based
on prediction model $\predictionmodel$, accuracy model $\accmodel$, and
threshold $t$.

The accuracy model $\accmodel$ used in this approach is similar to the binary
event forecaster used in the calibration scheme of Kuleshov and Liang \yrcite{kuleshov2015calibrated}, and is interchangeable with it in
the case where $\ametric(\hat y, y) \in \{0,1\}$.

\begin{algorithm}[h]
  \caption{ConfidentModel($x; \predictionmodel, \accmodel, t$)}
  \label{alg:confident_abstaining_model}
\begin{algorithmic}
   \STATE {\bfseries Parameters:}
	prediction model $\predictionmodel: \examples \rightarrow \labels$,
	accuracy model $\accmodel: \examples \rightarrow \reals$,
	threshold $t \in \reals$
   \STATE {\bfseries Input:} example $x \in \examples$
	\STATE {\bfseries return} $\predictionmodel(x)$ if $\hat q(x) \ge t$, else $\bot$
\end{algorithmic}
\end{algorithm}

\section {Cascade generation algorithm}

\newcommand{\answered}[0]{\ensuremath{\mathcal{A}}}
\newcommand{\cascade}[0]{\mathrm{cascade}}
\newcommand{\ctotal}[0]{\ensuremath{c_{\Sigma}}}
\newcommand{\pref}[0]{p_{\textrm{ref}}}
\newcommand{\numanswered}[1]{\ensuremath{\card{\answered\left(#1\right)}}}

Having created a set of abstaining models, we must next select a
sequence of abstaining models to use in our cascade.
Our goal is to generate a cascade that minimizes average cost as measured
on a validation set, subject to an accuracy constraint (e.g., requiring
that overall accuracy match that of some existing reference model).

We accomplish this using a greedy algorithm presented in \S\ref{sec:greedy}.
To make clear the flexibility of our approach, we present it as
a meta-algorithm parameterized by several functions:
\begin{enumerate}

\item An \emph{accuracy constraint} $\ac(p, R)$ determines whether a prediction model $p$ is sufficiently accurate on a set $R \subseteq \pairs$ of labeled examples.

\item A cost function $\cost(m, S)$ determines the cost of evaluating $m(x)$, possibly making use of intermediate results computed when evaluating each model in $S$.

\item An \emph{abstaining model generator} $\amg(R, S)$ returns a set of
	abstaining models, given the set $S$ of models that have already been added to the cascade by the greedy algorithm as well as the set $R$ of labeled examples remaining (those on which every model in $S$ abstains).

\end{enumerate}
Possible choices for these functions are discussed in sections \ref{sec:accuracy_constraints}, \ref{sec:cost_functions}, and \ref{sec:model_generators}, respectively.  \S\ref{sec:theoretical} presents theoretical results about the performance of the greedy algorithm.  \S\ref{sec:adaptive_policies} discusses how to
modify the greedy algorithm to return an adaptive policy tree rather than a
linear cascade, and
\S\ref{sec:architecture_search} discusses integrating the algorithm
with model architecture search.

\subsection {The greedy algorithm} \label {sec:greedy}

We now present the greedy cascade-generation algorithm.  As already
mentioned, the goal
of the algorithm is to produce a cascade that minimizes cost, subject to
an accuracy constraint.
The high-level idea of the algorithm is to find the abstaining model that
maximizes the number of non-$\bot$ predictions per unit cost, considering only
those abstaining models
that satisfy the accuracy constraint on the subset of examples for
which they return a prediction.
We then remove from the validation set the examples on which this model returns a prediction, and apply the same greedy rule to choose the next
abstaining model, continuing in this manner until no examples remain.

\newcommand{\mgood}[0]{M^{\textrm{accurate}}}
\newcommand{\museful}[0]{M^{\textrm{useful}}}

\begin{algorithm}[tbh]
  \caption{GreedyCascade($R; \ac, \cost, \amg$)}
  \label{alg:greedy}
\begin{algorithmic}
   \STATE {\bfseries Parameters:}
	accuracy constraint $\ac$,
	cost function $\cost$,
	abstaining model generator $\amg$
   \STATE {\bfseries Input:} validation set $R$ %
   \STATE Initialize
	$i = 1$,
	$R_1 = R$,
   \WHILE{$|R_i| > 0$}
	\STATE $M_i = \amg(R_i, m_{1:i-1})$
	\STATE $\museful_i = \set {m \in M_i: \answered(m, R_i) \neq \emptyset}$
	\STATE $\mgood_i = \set{m \in \museful_i: \ac(m, \answered(m, R_i))}$
	\STATE If $\mgood_i = \emptyset$, return $\bot$.
	\STATE Define $r_i(m) \equiv \frac {|\answered(m, R_i)|} {c(m, m_{1:i-1})}$
	\STATE $m_i = \textrm{argmax}_{m \in \mgood_i} \set { r_i(m) }$
	\STATE $R_{i+1} = R_i \setminus \answered(m_i, R_i)$
	\STATE $i = i + 1$
   \ENDWHILE
   \STATE {\bfseries return} $m_{1:i-1}$
\end{algorithmic}
\end{algorithm}

To define the algorithm precisely, we denote the set of examples on which
an abstaining model $m$ returns a prediction by
\[
	\answered(m, R) = \set{ (x, y) \in R: m(x) \neq \bot } \mbox{ .}
\]
Here and elsewhere, we use the shorthand $m_{j:k}$ to denote the sequence $\seq{m_i}_{i=j}^k$.
Algorithm~\ref{alg:greedy} gives pseudo-code for the algorithm using this notation.

Our greedy algorithm builds on earlier approximation algorithms
for min-sum set cover and related problems \cite{feige2004approximating,munagala2005pipelined,streeter2007combining}.
The primary difference is that our algorithm must worry about maintaining
accuracy in addition to minimizing cost.
We also consider a more general notion of cost,
reflecting the possibility of reusing intermediate computations.

\subsection{Accuracy constraints} \label {sec:accuracy_constraints}

We first consider the circumstances under which Algorithm~\ref{alg:greedy}
returns a cascade that satisfies the accuracy constraint.

Let $S = m_{1:k}$ be a cascade returned by the greedy algorithm, and let
$A_i = \answered(m_i, R_i)$, where $R_i$ is the set of examples remaining
at the start of iteration $i$.  Observe that $R_{i+1} = R_i \setminus A_i$,
which implies that $A_i$ and $A_j$ are disjoint for $j \neq i$.  Also,
because $R_1 = R$ and $R_k = \emptyset$, $\bigcup_i A_i = R$.

By construction, $\ac(m_i, A_i)$ holds for all $i$.  Because $S$ uses $m_i$ to
make predictions on examples in $A_i$, this implies $\ac(S, A_i)$ holds as well.
Thus, the accuracy constraint $\ac(S, R)$ will be satisfied so long as
$\ac(S, A_i) \ \forall i$ implies $\ac(S, \bigcup_i A_i)$.
A sufficient condition is that the accuracy constraint is \emph{decomposable},
as captured in the following definition.

\begin {definition}
An accuracy constraint $\ac$ is \emph{decomposable} if, for any two
disjoint sets $A, B \subseteq \examples \times \labels$,
$
	\ac(m, A) \wedge \ac(m, B) \implies \ac(m, A \cup B)
$.
\end {definition}

An example of a decomposable accuracy constraint is the MinRelativeAccuracy
constraint shown in Algorithm~\ref{alg:min_relative_accuracy}, which
requires that average accuracy according to some metric is at least
$\alpha$ times that of a fixed reference model.  %
When using this constraint, any cascade returned by
the greedy algorithm is guaranteed to have overall accuracy at least $\alpha$
times that of the reference model $\pref$.

We now consider the circumstances under which the greedy algorithm terminates
successfully (i.e., does not return $\bot$).  This happens so long as
$\mgood_i$ is always non-empty.  A sufficient condition is that
the accuracy constraint is \emph{satisfiable}, defined as follows.

\begin {definition}
An accuracy constraint $\ac$ is \emph{satisfiable} with respect to an
abstaining model generator $g$ and validation set $R$
if there exists a model $m^*$ that
(a) never abstains,
(b) is always returned by $g$, and
(c) satisfies $\ac(m^*, R_0) \ \forall R_0 \subseteq R$.
\end {definition}

The MinRelativeAccuracy constraint is satisfiable provided the reference model
$\pref$ is always among the models returned by the model generator $g$.

\newcommand{\qmin}[0]{\ensuremath{q_{\textrm{min}}}}
Note that MinRelativeAccuracy is \emph{not} the same as simply requiring a fixed
minimum average accuracy (e.g., 80\% top-1 accuracy).  Rather, the accuracy
required depends on the reference model's performance on the provided subset
$R_0$, which takes on many different values when running Algorithm~\ref{alg:greedy}.
A constraint that requires 
accuracy $\ge \qmin$ on $R_0$ is generally
not satisfiable, because $R_0$ might contain only examples that all models
misclassify.

\begin{algorithm}[t]
	\caption{MinRelativeAccuracy($p, R_0; \alpha, q, \pref$)}
   \label{alg:min_relative_accuracy}
\begin{algorithmic}
   \STATE {\bfseries Parameters:}
	$\alpha \in (0, 1]$,
	accuracy metric $q$, %
	prediction model $\pref$ %
   \STATE {\bfseries Input:} prediction model $p$, %
	validation set $R_0$ %
	\STATE Define $Q(p') = \sum_{(x, y) \in R_0} q(p'(x), y)$
	\STATE {\bfseries return} $\indicator{Q(p) \ge \alpha \cdot Q(\pref)}$
\end{algorithmic}
\end{algorithm}

\subsection {Cost functions} \label{sec:cost_functions}

We now consider possible choices for the cost function $\cost$.  In the simplest case, there is no reuse of computations and $\cost(m, S)$ depends only on $m$,
in which case we say the cost function is \emph{linear}.

To allow for computation reuse, we define a weighted, directed graph
with a vertex $v(m)$ for each model $m$, plus a special vertex $v_\emptyset$.
For each $m$, there is an edge $(v_\emptyset, v(m))$ whose weight is the
cost of computing $m(x)$ from scratch.  An edge $(v(m_1), v(m_2))$ with
weight $w$ indicates that $m_2(x)$ can be computed at cost $w$ if $m_1(x)$
has already been computed.
The cost function is then:
\begin{equation} \label{eq:graph_cost}
	\cost(m, S) = \min_{v \in V(S)}
\set { \mathrm{shortest\_path}(v, v(m))	}
\end{equation}
where $V(S) = \set{v(m): m \in S} \cup \set{v_\emptyset}$.

As an example, suppose we have a set of models $\set{m_i: 1 \le i \le D}$,
where $m_i$ makes predictions by computing the output of the first $i$
layers of a fixed neural network (e.g., a ResNet-like image classifier).
In this case, the graph is a linear chain whose $i$th edge has weight
equal to the cost of computing the output of layer $i$ given its input.

Equation \eqref{eq:graph_cost} can also be generalized in terms of hypergraphs
to allow reuse of multiple intermediate results.
This is useful
in the case of ensemble models, which take a
weighted average of other models' predictions.

\subsection{Abstaining model generators} \label {sec:model_generators}

We now discuss choices for the abstaining model generator $g$ used in
Algorithm~\ref{alg:greedy}.  Given a set $R_i$ of examples remaining in
the validation set, and a sequence $m_{1:i-1}$ of models that are already
in the cascade, $g$ returns a set of models to consider for the $i$th
stage of the cascade.

A simple approach to defining $g$ is to take a fixed set $P$ of prediction
models,
and for each one to return a ConfidentModel with the threshold
set just high enough to satisfy the accuracy constraint, as illustrated
in Algorithm~\ref{alg:confident_model_set}.

\newcommand{\accmodels}[0]{\ensuremath{\hat Q}}
\newcommand{\tmin}[0]{t_{\textrm{min}}}
\newcommand{\mfunc}[0]{\ensuremath{m}}

\begin{algorithm}[h]
	\caption{ConfidentModelSet($R; P, \accmodels, \ac$)}
   \label{alg:confident_model_set}
\begin{algorithmic}
   \STATE {\bfseries Parameters:} set $P$ of prediction models,
	set $\accmodels$ of accuracy models,
	accuracy constraint $\ac$
   \STATE {\bfseries Input:} validation set $R$
   \STATE Define $\ac^\top(m) = \ac(m, \answered(m, R))$
   \STATE Define $\mfunc(p, \accmodel, t) = \textrm{ConfidentModel}(\cdot; p, \accmodel, t)$
	\STATE Define $\tmin(p, \accmodel) = \min\set{t \in \extreals: \ac^\top(\mfunc(p, \accmodel, t))}$
	\STATE {\bfseries return} \set { \mfunc(p, \accmodel, \tmin(p, \accmodel))  : p \in P, \accmodel \in \accmodels }
\end{algorithmic}
\end{algorithm}

Another lightweight approach is to fit an ensemble model that makes use
of the already-computed predictions of the first $i-1$ models.  Assume that
each abstaining model $m_j$ has a backing prediction model $p_j$ that never
returns $\bot$.  For each $p$ in a fixed set $P$ of prediction models,
we fit an ensemble model $\bar p(x) = \beta_0 p(x) + \sum_{j=1}^{i-1} \beta_j p_j(x)$, where $\beta$ is optimized to maximize accuracy on the
remaining examples $R_i$.  Each $\bar p$ can then be converted to a
ConfidentModel in the same manner as above.

The most thorough (but also most expensive) approach is
to perform a search
to find a model architecture that yields the best benefit/cost ratio,
as discussed further in \S\ref{sec:architecture_search}.

\subsection{Theoretical results} \label {sec:theoretical}

In this section we provide performance guarantees for Algorithm~\ref{alg:greedy}, showing that under reasonable assumptions it produces a cascade that satisfies the accuracy constraint and has cost within a factor of 4 of optimal.  We also show that even in very special cases, the problem of finding a cascade whose
cost is within a factor $4-\epsilon$ of optimal is NP-hard for any $\epsilon > 0$.

As shown in \S\ref{sec:accuracy_constraints}, the greedy algorithm will return
a cascade that satisfies the accuracy constraint provided the
constraint is \emph{decomposable} and \emph{satisfiable}.  This is a fairly
weak assumption, and is satisfied by the MinRelativeAccuracy constraint given
in Algorithm~\ref{alg:min_relative_accuracy}.

We now consider the conditions the cost function must satisfy,
which are more subtle.
Our guarantees hold for all linear cost functions, as well
as a certain class of functions that allow for a limited form of computation
reuse.  To make this precise, we will use the following definitions.

\begin{definition}
A set $M^*$ of abstaining models \emph{dominates} a sequence $\seq{m_i}_{i=1}^k$ of
abstaining models with respect to a cost function $\cost$ if two conditions hold:
\begin {enumerate}
  \item $\sum_{m \in M^*} \cost(m, \emptyset) \le \sum_{i=1}^k \cost(m_i, m_{1:i-1})$, and
  \item for any $x \in \examples$, if $m_i(x) \neq \bot$ for some $i$, then
	  $m(x) \neq \bot$ for some $m \in M^*$.
\end {enumerate}
\end{definition}

If the cost function is linear, $\cost(m, S) = \cost(m, \emptyset) \ \forall m, S$, and any sequence of abstaining models is dominated by the corresponding set.

\begin{definition}
A cost function $\cost$ is \emph{admissible} with respect to a set of abstaining models $M$ if, for any sequence %
of models in $M$, there exists a set $M^* \subseteq M$ that dominates it.
\end{definition}

\newcommand{\mdom}[0]{m_{\textrm{dom}}}

A linear cost function is always admissible.
Cost functions of the form \eqref{eq:graph_cost} are admissible under certain
conditions.  A sufficient condition is that the graph defining the cost
function is a linear chain, and for each edge $(v(m_i), v(m_{i+1}))$,
$m_{i}(x) \neq \bot \implies m_{i+1}(x) \neq \bot$.
If the graph is a linear
chain but does not have this property, we can make the cost function
admissible by including additional models.  Specifically,
for each $k$, we add a model $m^*_k$ that computes the output of models
$m_1, m_2, \ldots, m_k$ in order (at cost $\cost(m^*_k, \emptyset) = \sum_{i=1}^k \cost(m_i, m_{1:i-1}) = \cost(m_k, \emptyset)$), and then returns the prediction (if any) returned by the model with maximum index.
The singleton set $\set{m^*_k}$ will then dominate any sequence composed of
models in
$\set{m_1, m_2, \ldots, m_k}$.
Similar arguments apply to graphs comprised of multiple linear chains.
(Such graphs arise if we have multiple deep neural networks, each of
which can return a prediction after evaluating only its first $k$ layers.)

We also assume $\cost(m, S) \le \cost(m, \emptyset) \ \forall m, S$ (i.e., reusing intermediate computations does not hurt).

To state our performance guarantees, we now introduce some additional notation.  For any cascade $S = m_{1:k}$ and
cost function $\cost$, let
\begin{equation} \label{eq:csigma}
  \ctotal(S) \equiv \sum_{i=1}^k \cost(m_i, m_{1:i-1})
\end{equation}
be the cost of computing the output of all stages.  For any example $x$
and cascade $S=m_{1:k}$,
let $\tau(x, S)$ be the cost of computing the output
of $S$, that is
\begin{equation} \label {eq:tau}
	\tau(x, S) = \sum_{i: m_j(x) = \bot \forall j < i} \cost(m_i, m_{1:i-1}) \mbox { .}
\end{equation}
Finally, for any set $M$ of models, we define $\answered(M, R) = \cup_{m \in M} \answered(m, R)$.

The following lemma shows that,
if the cost function is admissible, the number of examples that a cascade
can answer per unit cost is bounded by the maximum number of examples any
single model can answer per unit cost.  Theorem~\ref{thm:greedy} then
uses this inequality to bound the approximation ratio of the greedy algorithm.

\newcommand{\rmax}[0]{r^*}

\begin{lemma} \label {lem:rate}
For any set $R \subset \examples \times \labels$, any set $M$ of abstaining models,
any cost function $c$ that is admissible with respect to $M$, and any sequence $S$ of models in $M$,
\[
	\numanswered{S, R} \le \rmax \ctotal(S)
\]
where
\[
	\rmax = \max_{m \in M} \set { \frac {\numanswered{m, R}} {c(m, \emptyset) } } %
\]
and $\ctotal$ is defined in Equation~\eqref{eq:csigma}.
\end{lemma}
\begin{proof}
	For any $m \in M$, $\numanswered{m, R} \le \rmax \cdot \cost(m, \emptyset)$ by definition of $\rmax$.  Thus, for any set $M^* \subseteq M$,
\begin{align*}
	\numanswered{M^*, R} & \le \sum_{m \in M^*} \numanswered{m, R} \\
	& \le \rmax \sum_{m \in M^*} \cost(m, \emptyset) \mbox { .}
\end{align*}
Because $\cost$ is admissible with respect to $M$, there exists an $M^*$
	with $\numanswered{S, R} \le \numanswered{M^*, R}$ and
	$\sum_{m \in M^*} \cost(m, \emptyset) \le \ctotal(S)$.  Combining
	this with the above inequality proves the lemma.
\end{proof}

The proof of Theorem~\ref{thm:greedy} (specifically the proof of claim 2) is
along the same lines as the
analysis of a related greedy algorithm for generating a \emph{task-switching schedule} \cite{streeter2007combining}, which in turn built on an elegant geometric proof technique developed by Feige et al.\ \yrcite{feige2004approximating}.

\newenvironment{claim1}[1]{\par\noindent\underline{Claim 1:}\space#1}{}
\newenvironment{claim1proof}[1]{\par\noindent\underline{Proof of claim 1:}\space#1}{}
\newenvironment{claim2}[1]{\par\noindent\underline{Claim 2:}\space#1}{}
\newenvironment{claim2proof}[1]{\par\noindent\underline{Proof of claim 2:}\space#1}{}

\newcommand{\gci}[0]{\ensuremath{C_i}}
\newcommand{\greedy}[0]{\mathsf{GREEDY}}
\newcommand{\nri}[0]{n_i}
\newcommand{\nrip}[0]{n_{i+1}}
\newcommand{\oc}[1]{t_{#1}}
\newcommand{\opt}[0]{\mathsf{OPT}}
\newcommand{\sgreedy}[0]{\ensuremath{S_{\textrm{G}}}}
\newcommand{\sopt}[0]{\ensuremath{S^*}}

\newcommand{\tautotal}[0]{\ensuremath{\mathrm{T}}}

\begin{theorem} \label {thm:greedy}
Let $\sgreedy = \mathrm{GreedyCascade}(R; \ac, \cost, \amg)$.
	Let $M$ be a set of models such that $M \subseteq \mgood_i$ for all $i$,
	and $\cost$ is admissible with respect to $M$,
where $\mgood_i$ is defined as in Algorithm~\ref{alg:greedy}.
	Define $\tautotal(S) \equiv \sum_{(x, y) \in R} \tau(x, S)$, where $\tau$ is defined in Equation~\eqref{eq:tau}.
Then,
\[
	\greedy \le 4 \cdot \opt
\]
where $\greedy = \tautotal(\sgreedy)$, $\opt = \min_{S \in M^{\infty}} \set {\tautotal(S)}$.
\end {theorem}
\begin{proof}

We first introduce some notation.  Let $\sgreedy = m_{1:k}$,
let $\nri = \card{R_i}$ denote the number of examples remaining at the start
of the $i$th iteration of the greedy algorithm, and let $\gci = \cost(m_i, m_{1:i-1})$ be the cost of the abstaining model selected in the $i$th iteration.  Let $r^*_i = \frac {\nri - \nrip} {\gci}$ be the maximum benefit/cost ratio on iteration $i$.
Let $\sopt$ be an optimal cascade.

\begin{claim1}
For any $i$, there are at least $\frac {\nri} {2}$ examples with $\tau(x, \sopt) \ge \frac {\nri} {2 r^*_i}$.
\end {claim1}
\begin{claim1proof}
	Let $t = \frac {\nri} {2 r^*_i}$, and let $S_0$ be the maximal prefix of $\sopt$ satisfying $\ctotal(S_0) \le t$.
	Any example $x \in R_i \setminus \answered(S_0, R_i)$ must have $\tau(x, \sopt) \ge t$.  Thus, it suffices to show $\numanswered{S_0, R_i} \le \frac {\nri} {2}$.
	By Lemma~\ref{lem:rate}, $\numanswered{S_0, R_i} \le r^* t$,
	where $r^* = \max_{m \in M} \set{ \frac {\numanswered{m, R_i}} {\cost(m, \emptyset)}  }$.  By assumption, $c(m, \emptyset) \ge c(m, m_{1:i-1})$ for all $m \in M$,
	which implies $\rmax \le \max_{m \in M} \set {\frac {\numanswered{m, R_i}} {\cost(m, m_{1:i-1})}} \le r^*_i$,
	where the last inequality uses the fact that $M \subseteq \mgood_i$.
Thus, $\numanswered{S_0, R_i} \le t r^* \le t r^*_i = \frac {\nri} {2}$.
\end{claim1proof}

\begin{claim2}
$\greedy \le 4 \cdot \opt$.
\end{claim2}
\begin{claim2proof}
The total cost of the cascade $\sgreedy$ is the
sum of the costs associated with each stage, that is
\[
	\greedy = \sum_{x \in R} \tau(x, \sgreedy)
	= \sum_{i=1}^k \nri \gci \mbox { .}
\]

To relate $\opt$ to this expression, let $\set{\oc{j}}_{j=1}^n$ be the
	sequence that results from sorting the costs $\set{\tau(x, \sopt): x \in R}$
	in descending order.
Assuming for the moment that $\nri$ is even for all $i$, let
$J(i) = \set{j: \frac {n_{i+1}} {2} < j \le \frac {n_i} {2}}$.
Because $\seq{\oc{j}}$ is non-increasing, and $n_1 = n$ while $n_k = 0$,
\[
	\opt = \sum_{j=1}^n \oc{j}
	\ge \sum_{j=1}^{\frac n 2} \oc{j}
	= \sum_{i=1}^k \sum_{j \in J(i)} \oc{j}
	\mbox { .}
\]
Thus, to show $\opt \ge \frac 1 4 \greedy$, it suffices to show that for
any $i$,
\begin{equation} \label {eq:wts}
	\sum_{j \in J(i)} \oc{j} \ge \frac 1 4 n_i \gci \mbox { .}
\end{equation}
To see this, first note that for any $i$,
\[
	\sum_{j \in J(i)} \oc{j} \ge \card{J(i)} \cdot
\oc{\left(\frac {\nri} {2}\right)}
	= \frac {r^*_i \cdot \gci} {2} \cdot \oc{\left(\frac {\nri} {2}\right)}
\]
By claim 1, $\oc{\left(\frac {\nri} {2}\right)} \ge \frac {\nri} {2 r^*_i}$.  Combining
this with the above inequality proves \eqref{eq:wts}.

Finally, if $\nri$ is odd for some $i$, we can apply the above argument
to a set $R'$ which contains two copies of each example in $R$,
in order to prove the equivalent inequality
$ 2 \greedy \le 4 \cdot 2 \opt$.
\end{claim2proof}
\end{proof}

Finally, we consider the computational complexity of the optimization problem
Algorithm~\ref{alg:greedy} solves.  Given a validation set $R$, set of
abstaining models $M$, and accuracy constraint
$\ac$, we refer to the problem of finding a minimum-cost cascade $S$
that satisfies the accuracy constraint
as {\sc Minimum Cost Cascade}.
This problem is NP-hard to approximate even
in very special cases, as summarized in Theorem~\ref{thm:hard}.

\renewcommand{\proofname}{Proof (sketch)}
\begin {theorem} \label{thm:hard}
For any $\epsilon > 0$, it is NP-hard to obtain an approximation ratio
	of $4-\epsilon$ for {\sc Minimum Cost Cascade}.  This is true even in
	the special case where:
	(1) the cost function always returns 1, and
	(2) the accuracy constraint is always satisfied.
\end {theorem}
\begin {proof}
	The theorem can be proved using a reduction from {\sc Min-Sum Set Cover} \cite {feige2004approximating}.  In the reduction, each element $e$ in the {\sc Min-Sum Set Cover} instance becomes an example $x_e$ in the validation set, and each set $Z$
	becomes a prediction model $m_Z$ where $m_Z(x_e) = \bot$ iff.\ $e \notin Z$.  The cost function in the {\sc Minimum Cost Cascade} instance always returns 1, and the accuracy constraint is always
	satisfied.
\end {proof}
\renewcommand{\proofname}{Proof}

\subsection {Adaptive policies} \label {sec:adaptive_policies}

In this section we discuss how to modify Algorithm~\ref{alg:greedy} to
return an adaptive policy tree rather than a linear cascade.

The greedy algorithm for set covering can be modified to produce an
adaptive policy \cite{golovin2011adaptive}, and a similar approach can be
applied Algorithm~\ref{alg:greedy}.  The resulting algorithm is similar
to Algorithm~\ref{alg:greedy}, but instead of building up a list of
abstaining models it builds a tree, where each node of the tree is labeled
with an abstaining model and each edge is labeled with some feature of the
parent node's output (e.g., a discretized confidence score).

If the $\answered$ function satisfies a technical condition
called \emph{adaptive monotone submodularity}, the resulting algorithm
has an approximation guarantee analogous to the one stated
in Theorem~\ref{thm:greedy}, but with respect to the best adaptive policy
rather than merely the best linear sequence.  This can be shown by combining
the proof technique of Golovin and Krause \yrcite{golovin2011adaptive}
with the proof of Theorem~\ref{thm:greedy}.
Unfortunately, the $\answered$ function is not guaranteed to have
this property in general.
However, it can be shown that
the adaptive version of Algorithm~\ref{alg:greedy} still has the guarantees
described in Theorem~\ref{thm:greedy} (i.e., adaptivity does not hurt).

\subsection {Greedy architecture search} \label {sec:architecture_search}

The GreedyCascade algorithm can be integrated with
model architecture search in multiple ways.  One way would be
to simply take all the models evaluated by an architecture search as input
to the greedy algorithm.  A potentially much more
powerful approach is to use architecture search as the model generator
$g$ used in the greedy algorithm's inner loop.

With this approach, there is one architecture search for each stage of
the generated cascade.  The goal of the $i$th search is to maximize the
benefit/cost ratio criterion used on the $i$th iteration of the greedy
algorithm (subject to the accuracy constraint).
Because the $i$th search only needs to consider examples not already
classified by the first $i-1$ stages, later searches have potentially lower
training cost.
Furthermore, the $i$th
model can make use of the intermediate layers of the first $i-1$ models
as input features, allowing computations to be reused across
stages of the cascade.

\section {Experiments}

\begin{table*}[t]
  \label{tab:cascade}
	\caption{Cascade of pre-trained MobileNet \cite{howard2017mobilenets} models.
	}
  \centering
	\begin{small}
	\begin{sc}
	\begin{tabular}{lllp{2cm}p{2cm}p{3cm}}
    \toprule
		Stage & Image size & \# Mults & {\centering Confidence threshold \newline (logit gap)} & \%(examples classified) & Accuracy \newline (on examples classified by stage) \\
    \midrule
		1 & 128 x 128 & 49M & 1.98 & 40\% & 88\% \\  %
		2 & 160 x 160 & 77M & 1.67 & 16\% & 73\% \\  %
		3 & 160 x 160 & 162M & 1.23 & 18\% & 62\% \\  %
		4 & 224 x 224 & 150M & 1.24 & 7\% & 45\% \\  %
		5 & 224 x 224 & 569M & $-\infty$ & 19\% & 45\% \\  %
    \bottomrule
  \end{tabular}
  \end{sc}
  \end{small}
\end{table*}

In this section we evaluate our cascade generation algorithm by applying it to
state-of-the-art pre-trained models for the
ImageNet classification task.
We first examine the efficacy of the abstention rules described
in \S\ref{sec:abstaining_models}, then we evaluate the full cascade-generation
algorithm.

\subsection {Accuracy versus abstention rate}

\newcommand{\avg}[0]{\ensuremath{\mathrm{avg}}}
\newcommand{\xp}{\ensuremath{\tilde x}}
\newcommand{\yp}{\ensuremath{\tilde y}}

As discussed in \S\ref{sec:abstaining_models}, we decide whether a model should
abstain from making a prediction by training a second model to predict its
accuracy on a given example, and checking whether predicted accuracy falls
below some threshold.

For our ImageNet experiments, we take top-1 accuracy as the accuracy metric, and
predict its value based on a vector of features derived from
the model's predicted class probabilities.
We use as features
(1) the entropy of the vector,
(2) the maximum predicted class probability, and
(3) the gap between the first and second highest predictions in logit space.
Our accuracy
model $\accmodel$
is fit using logistic regression on a validation set of 25,000 images.

\begin{figure} [h]
\begin{center}
  \includegraphics[width=3in]{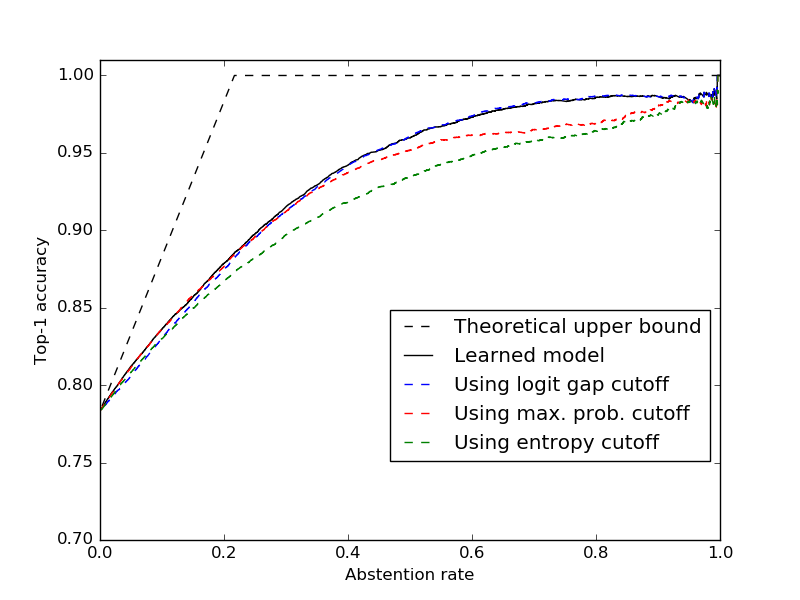}
	\caption{Accuracy vs. abstention rate for Inception v3.}
  \label{fig:abstaining_models}
\end{center}
\end{figure}

Figure~\ref{fig:abstaining_models} illustrates the tradeoffs between accuracy and
response rate that can be achieved by applying this rule to
Inception-v3, measured on a second disjoint validation set of 25,000 images.
The horizontal
axis is the fraction of examples on which Inception-v3 returns $\bot$,
and the vertical axis is top-1 accuracy on the remaining examples.
For comparison, we also show for each feature the tradeoff
curve obtained by simply thresholding the raw feature value.
We also show the theoretically
optimal tradeoff curve that would be achieved
using an accuracy model that predicts top-1 accuracy perfectly
(in which case
we only return $\bot$ on examples Inception-v3 misclassifies).

Overall, Inception-v3 achieves 77\% top-1 accuracy.  However, if we set the
logit gap cutoff threshold appropriately, we can achieve over 95\% accuracy
on the 44\% of examples on which the model is most confident (while returning
$\bot$ on the remaining 56\%).  Perhaps surprisingly, using the learned
accuracy model gives a tradeoff curve almost identical to that obtained
by simply thresholding the logit gap.

\subsection {Cascades} \label {sec:cascades}

Having shown the effectiveness of our abstention rules, we now
evaluate our cascade-generation algorithm on a pool of state-of-the-art
ImageNet classification models.  Our pool consists of
23 models released as part of the TF-Slim library \cite{silberman2016tfslim}.
The pool contains
two recent NasNet models produced by neural architecture search
\cite{zoph2017learning},
five models based on the Inception architecture \cite {szegedy2016rethinking},
and all 16 MobileNet models \cite{howard2017mobilenets}.
We generate abstaining models by thresholding the logit gap value.

For each model $p$, and each
$\alpha \in \set{1 - \frac {i} {100}: 0 \le i \le 5}$,
we
used Algorithm~\ref{alg:greedy} to generate a cascade
with low cost, subject to the constraint that accuracy was at least $\alpha$
times that of $p$.  We use number of multiplications as the cost.

When using Algorithm~\ref{alg:greedy} it is important
to use examples not seen during training, because statistics such as logit
gap are distributed very differently for them.
We used 25,000 images from the
ILSVRC 2012 validation set \cite{russakovsky2015imagenet} to
run the algorithm, and report results on the
remaining 25,000 validation images.\footnote{The cascade returned by the greedy algorithm always returns a non-$\bot$ prediction on unseen test images, because the final stage of the cascade always uses a confidence threshold of $-\infty$.}

\begin{figure}[tbh]
\begin{center}
  \includegraphics[width=3in]{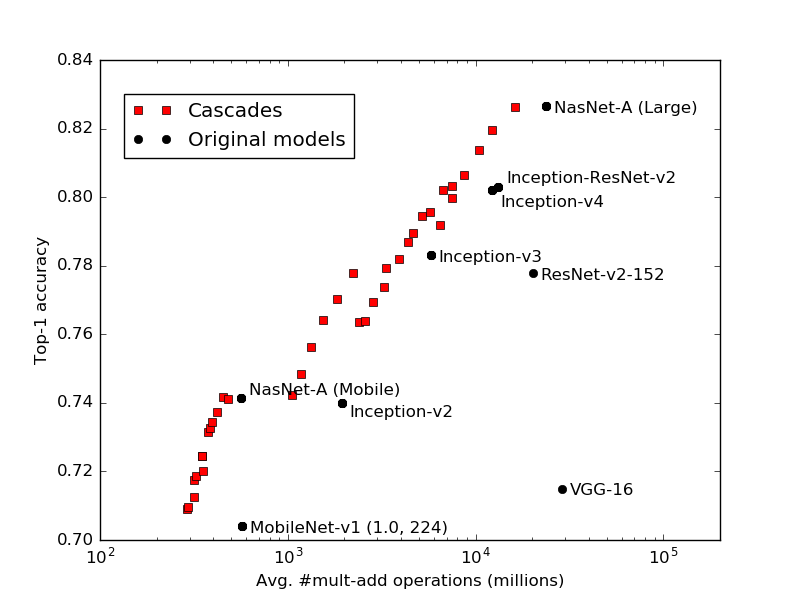}
  \caption{Cascades of pre-trained ImageNet models.}
  \label{fig:num_macs}
\end{center}
\end{figure}

Figure~\ref{fig:num_macs} shows the tradeoffs we achieve between accuracy
and average cost.
Relative to the large (server-sized) NasNet model, we obtain a 1.5x reduction
with no loss in accuracy.  Relative to
Inception-v4, one of the cascades obtains a 1.8x cost reduction with a
no loss in accuracy, while another obtains a 1.2x cost reduction with
a 1.2\% \emph{increase} in accuracy.  Relative to the largest MobileNet model,
we achieve a 2x cost reduction with a 0.5\% accuracy gain.

We now examine the structure of an auto-generated cascade.
Table 1 shows a cascade generated using a pool of 16 MobileNet models,
with the most accurate MobileNet model as the reference model.
The $i$th row of the table describes the (model, rule)
pair used in the $i$th stage of the cascade.
The cascade has several intuitive properties:

\begin{enumerate}
	\item \emph{Earlier stages use cheaper models.}  Model used in earlier stages
		of the cascade have fewer parameters, use fewer
		multiplies, and have lower input image
		resolution.
	\item \emph{Cheaper models require higher confidence.}  The minimum
		logit gap required to make a prediction is higher for earlier
		stages, reflecting the fact that cheaper models
		must be more confident in order to achieve
		sufficiently high accuracy.
	\item \emph{Cheaper models handle easier images.}
Although overall model accuracy increases in later stages,
		accuracy on the subset of images actually classified by each stage is strictly \emph{decreasing} (last column).
		This supports the idea that easier images are allocated to
		cheaper models.
\end{enumerate}

\subsection {Cascades of approximations} \label {sec:cascades-of-approximations}

A large number of techniques have been developed for reducing the cost
of deep neural networks via postprocessing.  Such techniques include
quantization, pruning of weights or channels, and tensor factorizations (see
\cite{han2015deep} and references therein for further discussion).
In this section, we show how these techniques can be used to
generate a larger pool of approximated models, which can then be used
as input to our cascade-generation algorithm in order to achieve further
cost reductions.  This also provides a way to make use of the
cascade-generation algorithm in the case where only a single pre-trained model
is available.

For these experiments, we focus on quantization of model parameters as
the compression technique.  For each model $m$, and each number of bits
$b \in \{1, 2, \ldots, 16\}$, we generate a new model $m_b$ by quantizing
all of $m$'s parameters to $b$ bits.  This yields a pool of $23 \cdot 16 = 368$
quantized models, which we use as input to the cascade-generation algorithm.
Cost is the number of bits read from memory when
classifying an example.  Aside from these two changes, our experiments are
identical to those in \S\ref{sec:cascades}.

\begin{figure}[h]
\begin{center}
  \includegraphics[width=3in]{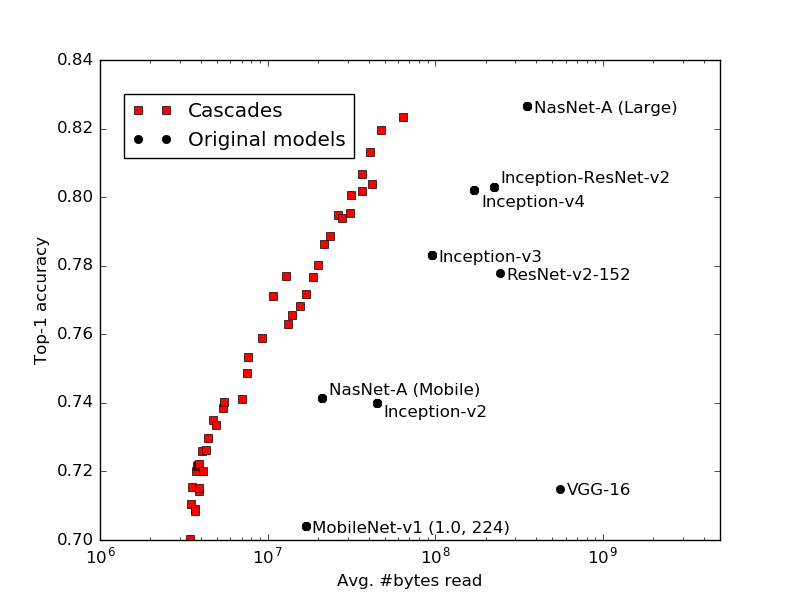}
	\caption{Cascades of quantized ImageNet models.}
  \label{fig:approx_num_bits}
\end{center}
\end{figure}

Figure~\ref{fig:approx_num_bits} shows accuracy as a function of
the average number of bits that must be fetched from memory in order to
classify an example.
Though the cascades generated in \S\ref{sec:cascades} (which were optimized for
number of multiplications) do not consistently improve on average memory I/O,
the cascades of approximations reduce it by up to a factor of 6 with
no loss in accuracy.
Perhaps
surprisingly, the cascades of approximations also
offer improvements in average-case number of multiplications similar to those
shown in Figure~\ref{fig:num_macs}, despite this not
being an explicit component of the cost function.

\section{Related work}

The high-level goal of our work is to reduce average-case inference cost
by spending less computation time on easier examples.  This subject is the
topic of a vast literature, with many distinct problem formulations spread
across many domains.

Within the realm of computer vision,
cascaded models received much attention following the seminal
work of Viola and Jones \yrcite{viola2001rapid}, who used a cascade of
increasingly more expensive features to create a fast and accurate face
detector.
In problems such as speech recognition and machine translation, inference
typically involves a heuristic search through the large space of possible
outputs, and the cascade idea can be used to progressively filter the
set of outputs under consideration \cite {weiss2010structured}.

Recent work has sought to apply the cascade idea to deep
networks.  Most research involves training an adaptive model end-to-end
\cite{graves2016adaptive,guan2017energy,hu2017anytime,huang2017multiscale}.
Though end-to-end training is appealing, it is sensitive to the choice of
model architecture.  Current approaches for image classification are based on
ResNet architectures, and do not achieve results competitive with the latest
NasNet models on ImageNet.

Another way to produce adaptive deep networks is to apply postprocessing to a
pool of pre-trained models, as we have done.  To our knowledge, the only
previous work that has taken this route is that of Bolukbasi et al.
\yrcite{bolukbasi2017adaptive}, who also present results for ImageNet.
In contrast to the greedy approximation algorithm presented in this work,
their approach does not have good performance in the worst case, and also
requires the pre-trained input models to be arranged into a directed acyclic
graph a priori as opposed to learning this structure as part of the
optimization process.

Finally, as already mentioned, our greedy approximation algorithm builds on
previous greedy algorithms for min-sum set cover and related problems
\cite{feige2004approximating,munagala2005pipelined,streeter2007combining}.

\section {Conclusions}

We presented a greedy meta-algorithm that can generate a cascaded model
given a pool of pre-trained models as input, and proved that the algorithm
has near-optimal worst-case performance under suitable assumptions.
Experimentally, we showed that cascades generated using this algorithm
significantly improve upon state-of-the-art ImageNet models in terms of both
average-case number of multiplications and average-case memory I/O.

Our work leaves open several promising directions for future research.
On the theoretical side, it remains an open problem to come up with
more compelling theoretical guarantees for adaptive policies as opposed to
linear cascades.
Empirically, it would be very interesting to
incorporate architecture search
into the inner loop of the greedy algorithm, as discussed in \S\ref{sec:architecture_search}.

\bibliography{cascades}
\bibliographystyle{icml2018}

\end{document}